\documentclass[10pt,letterpaper]{article}
\pdfoutput=1%For ARXIV
\usepackage{pgfplots}
\pgfplotsset{compat=1.14}
\usepackage{PRIMEarxiv}

\usepackage{graphicx}
\usepackage{caption}
\usepackage{amsmath}
\usepackage{amssymb}
\usepackage{booktabs}

\usepackage{amsthm}
\newtheorem{theorem}{Theorem}
\usepackage{algorithm}
\usepackage{algorithmic}
\usepackage{amssymb}
\usepackage{float}
\usepackage[T1]{fontenc}
\usepackage{amsmath}
\usepackage{booktabs}
\usepackage{multirow}
\usepackage[pagebackref,breaklinks,colorlinks]{hyperref}

\usepackage[capitalize]{cleveref}
\crefname{section}{Sec.}{Secs.}
\Crefname{section}{Section}{Sections}
\Crefname{table}{Table}{Tables}
\crefname{table}{Tab.}{Tabs.}

\begin{document}
 
\title{How to Construct Energy for Images ?\\Denoising Autoencoder Can Be Energy Based Model}

\author{Weili Zeng\\
Shanghai Jiao Tong University\\
{\tt\small zwl666@sjtu.edu.cn}
}
\maketitle

%%%%%%%%% ABSTRACT
\begin{abstract}
Energy-based models parameterize the unnormalized log-probability of data samples, but there is a lack of guidance on how to construct the ``energy''. In this paper, we propose a Denoising-EBM which decomposes the image energy into ``semantic energy'' and ``texture energy''. 

We define the ``semantic energy'' in the latent space of DAE to model the high-level representations, and define the pixel-level reconstruction error for denoising as ``texture energy''. Inspired by score-based model, our model utilizes multi-scale noisy samples for maximum-likelihood training and it outputs a vector instead of a scalar for exploring a larger set of functions during optimization. After training, the semantics are first synthesized by fast MCMC through ``semantic energy'', and then the pixel-level refinement of semantic image will be performed to generate perfect samples based on ``texture energy''. Ultimately, our model can outperform most EBMs in image generation. And we also demonstrate that Denoising-EBM has top performance among EBMs for out-of-distribution detection.
\end{abstract}

%%%%%%%%% BODY TEXT
\section{Introduction}
\label{intro}

Deep generative learning plays an important role in machine learning, it has a very wide range of applications in image~\cite{ref1,ref2,ref3}, video~\cite{ref4}, distribution alignment across domains~\cite{ref10,ref11,ref12}, and out-of-distribution detection~\cite{ref15,ref16}. EBMs are attractive for their few require constraints on the network architecture, and they also have better robustness and generalization~\cite{ref23} due to samples with high probability under the model distribution but low probability under the data distribution will be directly penalized during training~\cite{ref24}. And a growing body of work has also shown that EBMs~\cite{ref25,ref26} can generate samples comparable to GAN.

EBM models the data density via Gibbs distribution, where the energy function $E(x)$ maps input variables $x$ to energy scalars. In statistical physics, $x$ is usually the state of a system (e.g., ferromagnetic substance, a cup of water, gas), then $E(x)$ is a function related to kinetic or internal energy of the system at state $x$. Therefore, when describing image with an energy function, we hope its construction can be guided under some physical meaning. Moreover, traditional energy function based on the bottom-up network structure directly maps the input to the energy scalar, it may not preserve all the detailed information needed to accurately parameterize the unnormalized log-probability.
And a critical drawback of training EBM is that each iteration requires the use of Markov Chain Monte Carlo (MCMC) to generate negative samples, which is computationally expensive, especially for high-dimensional data.

To address the above issues, this paper induces a new energy based framework Denoising-EBM and defines image energy from two aspects. Intuitively, for an image, people will first understand it from a semantic level and know what it is. After that, people look at its texture to see what it has. Enlightened by this, in this work, we assume that the ``energy" of an image consists of ``semantic energy'' (semantic information) and ``texture energy'' (texture information). In the deep convolutional neural network, the deeper network representation corresponds to the higher-level semantic information of the input image, so we exploit the rich high-level semantics in the latent space formed by the DAE to learn ``semantic energy''. Furthermore, the texture information of image should be captured at pixel level. Inspired by the work~\cite{ref28,ref29}, we define the ``texture energy'' through the denoising reconstruction error. The learning of the energy function is based on an U-Net~\cite{ref30} structure, with many short-cut connections specifically designed to propagate fine details from the inputs x to the high dimensional output.

We use the maximum likelihood to train our model, and in Section \ref{mle} we will show that the training objective function can finally be divided into four parts: a) Optimizing denoising reconstruction. b) Optimizing semantic reference. c) Optimizing semantic marginal likelihoods in latent space. d) Optimizing semantic-based conditional likelihoods in data space. Performing fast MCMC sampling first in the low-dimensional latent space can effectively save the computational resources of training and generation. At the same time, semantic-based conditional sampling makes the initial distribution of MCMC in data space closer to the model distribution that needs to be sampled, which benefit the convergence of MCMC chain and stablize training. After training, we generate semantics by fast MCMC in the semantic latent space first and transfers them back to the data space through semantic inference. Then we perform pixel-level refinement on the preliminary semantic image based on ``texture energy'' to generate the finished samples. In Section \ref{geo}, we also give the relevant geometric interpretation of our framework. Concretely, our contributions can be summarized below:

1) We propose a new energy based framework Denoising-EBM, where the image energy is decomposed to ``semantic energy'' and ``texture energy''. And we parameterize the ``energy'' through U-net with vector-output network instead of scalar-output.

2) Our method can take full advantage of the DAE structure to speed up sampling and stabilize maximum likelihood training.

3) We provide a geometric perspective of Denoising-EBM to analyze its training and sampling process.

4) Denoising-EBM can outperform most EBMs in image generation and has top performance among EBMs for out-of-distribution detection.

\section{Related work}
\label{related}
Our work is first related to FRAME ~\cite{ref31}, where the energy is defined as the response of input image and the Gabor filter. It also shows that the global texture information needs to be fully considered to construct the energy function of the image. Then with the development of deep learning, plenty of works~\cite{ref23,ref33,ref34} started to use deep convolutional neural networks with bottom-up structure to learn energy functions that map input images to energy scalars. However, due to the long mixing time of MCMC sampling for high-dimensional data, actually, it is difficult to implement and requires some practical tricks, such as replay buffer~\cite{ref23} and coarse-to-fine strategy~\cite{ref34}.

To alleviate the sampling complexity of high-dimensional data, ~\cite{ref55,ref56} introduced an adversarial learning framework with additional generator to learn the distribution under the energy model. However, maximizing generator entropy is a tricky problem. On the other hand, MCMC sampling can be performed in the latent space instead of the data space. Since the latent space is lower dimensional and generally unimodal, the mixing of MCMC would be more efficient. This often results in the need to jointly sample the latent space and data space  ~\cite{ref35,ref36,ref37,ref38,ref46}.

Our work also related to denoising auto-encode~\cite{ref39} and denoise score matching~\cite{ref40}, which was further developed by~\cite{ref41} and~\cite{ref20} for learning from data samples corrupted with multiple levels of noise. The difference between them is that the network output is a scalar function in~\cite{ref41} but~\cite{ref20} learn the score functions (the gradients of the energy functions) directly, instead of using the gradients of learned energy functions as in EBM. Such unconstrained score models are not guaranteed to output a conservative vector field, meaning they do not correspond to the gradient of any function, unlike EBM that are obtained through explicitly differentiating a parameterized energy function~\cite{ref29}. But since the unconstrained score model output is a vector, it does not always have to be the derivative of a scalar function during optimization, so it can explore a larger set of functions during optimization than the constrained model~\cite{ref41}.

\section{Preliminary}
There are two important components inside our model: i) Denoising auto-encoder, ii) Energy based model. In this section, we present the backgrounds of Denoise auto-encoder (DAE) and energy based model (EBM), which will serve as foundations of the proposed framework.
$\mathbf{Notation}$: let $x$ and $z$ be two random variables with joint probability density $p(x,z)$, with marginal distributions $p(x)$ and $p(z)$. We will use the following notation:

1)Entropy: $H(x)=\mathbb{E}_{p(x)}[-\log p(x)]$.

2)Conditional entropy: $H(z | x)=\mathbb{E}_{p(x, z)}[-\log p(z | x)]$.

3)KL divergence: $D_{KL}(p \| q)=\mathbb{E}_{p(x)}\left[\log \frac{p(x)}{q(x)}\right]$.

4)Mutual Information: $I(x,z)=H(x)-H(x|z)$

5)$ \delta$ function: {$\delta_{u}(v)= \left\{ \begin{array}{ll}
0 & u \neq v\\
1 & u = v
\end{array} \right.$}

\subsection{Denoising Auto-encoder}
Suppose that our dataset consists of i.i.d. samples $\left\{x_i \in \mathbb{R}^d\right\}_{i=1}^N$ from distribution $q(x)$. And the noise $\epsilon$ is Gaussian noise: $N(\epsilon)=\mathcal{N}(0,\sigma^{2}I_{d})$, then the perturbed data $\tilde{x}$ can be expressed as $\tilde{x}=x+\epsilon$ with distribution $\tilde{x} \sim q_{\sigma}(\tilde{x} | x)$. The DAE~\cite{ref39} is to train a denoised reconstruction function $r=g_{\beta} \circ f_{\alpha}$. Specifically, $f_{\alpha}: \mathcal{X} \to \mathcal{Z}$ is the encoder, which is a projection function from data space $\mathcal{X}$ to latent space $\mathcal{Z}$, and $g_{\beta}: \mathcal{Z} \to \mathcal{X}$ is a decoding function on the representation space, mapping back latent $z$ to the data space. Based on the Markov structure: $x \to \tilde{x} \to z$, we can define the observed joint distribution $q(x,\tilde{x},z)=q(x)q_{\sigma}(\tilde{x} | x)q_{\alpha}(z | \tilde{x})$ where $q_{\alpha}(z | \tilde{x})=\delta_{f(\tilde{x})}(z)$. The DAE is trained to minimize the following denoising criterion:
\begin{align}\nonumber
\mathcal{L}_{D A E}&=\mathbb{E}_{q(x, \tilde{x})} \mathbb{E}_{q_\alpha(z | \tilde{x})} \log q_\beta(x | z)\\
&=\mathbb{E}_{\tilde{x} \sim q_{\sigma}(\tilde{x})}\left[\|r(\tilde{x},\sigma)-x\|^{2}\right],\label{equ1}\\ 
q_{\sigma}(\tilde{x})&=\int q(x) q_{\sigma}(\tilde{x} | x) d x. \nonumber
\end{align}
\begin{theorem}
\label{theorem1}
\cite{ref41} If we train a DAE using equation \ref{equ1}, the optimal reconstruction function $r^{*}(\tilde{x})$ can be expressed in the following form as $\sigma \to 0$:
\begin{align}\nonumber
r^{*}(\tilde{x},\sigma)&=\arg \cdot \min _{r} \mathbb{E}_{\tilde{x} \sim q_{\sigma}(\tilde{x})}\left[\|r(\tilde{x})-x\|^{2}\right] \\
&=\tilde{x}+\sigma^{2} \nabla_{x} \log q_{\sigma}(\tilde{x}).\label{equ2}
\end{align}
\end{theorem}
Theorem \ref{theorem1} means that the optimal DAE for additive Gaussian noise can be calculated explicitly, and the result is related to actual score. It also tells us that $r^{*}(x)-x$ is actually an estimate of the actual score $\nabla_{x} \log q(x)$ as $\sigma \to 0$.
\begin{theorem}
\label{theorem2}
Consider $x$ be a random variable of train dataset over space $\mathcal{X}$ with $x \sim q(x)$. $z=f(x)$ is the latent representation over space $\mathcal{Z}$. When training a DAE, we are also maximizing a lower bound on mutual information $I(x,z)$.
\end{theorem}
\begin{proof}
\begin{align*}
&\mathbb{E}_{q(x, \tilde{x})} \mathbb{E}_{q_\alpha(z | \tilde{x})} \log q_\beta(x | z)\\
&=\mathbb{E}_{q(z, x)} \log q_\beta(x | z)\\
&\leq \mathbb{E}_{q(z, x)} \log q_\beta(x | z)+H_q(x)\\
&=I_q(x, z)\qedhere
\end{align*}
\end{proof}
\subsection{Energy Based Model}
Consider that we have the observed example $x_{1}, x_{2},\ldots, x_{N} \sim q(x)$, inspired from a physical point of view or the principle of maximum entropy. An energy based model (EBM) aims to approximates $q(x)$ with a Gibbs model, defined as follows:
\begin{equation}
p_{\theta}(x)=\frac{1}{Z(\theta)} \exp \left(-U_{\theta}(x)\right),
\label{equ3}
\end{equation}
where $U_{\theta}$ is the energy function with parameter $\theta$ and $Z(\theta)=\int_{x} \exp \left(-U_{\theta}(x)\right) d x$ is the partition function. There is no constrains on the particular form of $U_{\theta}(x)$. To find the parameter $\theta$, the model is optimized using the MLE. Objective function of Maximum Likelihood Estimation(MLE) learning is:
\begin{equation}
L(\theta)=\mathbb{E}_{q(x)}\left[-\log p_{\theta}(x)\right] \approx \frac{1}{N} \sum_{i=1}^{N}-\log p_{\theta}\left(x_{i}\right).
\end{equation}
The gradient of $L(\theta)$ is:
\begin{equation}
\begin{aligned}
\nabla_{\theta}L(\theta)&=\mathbb{E}_{q(x)}\left[\nabla_{\theta} U_{\theta}(x)\right]-\mathbb{E}_{p_{\theta}(x)}\left[\nabla_{\theta} U_{\theta}(x)\right]
\end{aligned}
\end{equation}
where $p_{\theta}(x)$ is approximated by Markov chain Monte Carlo (MCMC). Langevin Monte Carlo (LMC) is an MCMC method for obtaining random samples from probability distributions that are difficult to sample directly, which iterates
\begin{equation}
x_{t+1}=x_{t}-\eta \nabla_{x} U_{\theta}\left(x_{t}\right)+\sqrt{2 \eta} e_{t}, \quad e_{t} \sim \mathcal{N}(0, I)
\end{equation}
where $\eta$ corresponds to the step-size. As $x \to 0$ and $t \to \infty$, the distribution of $x_t$ converges to $p_{\theta}(x)$. 

The convergence of MCMC is independent of its initialization in theory, but the initialization method can be crucial in practice. Especially, when we sample in high-dimensional and multi-modal distributions, the initial state may deviate significantly from $p_{\theta}(x)$, which would lead to difficult convergence of MCMC sampling and poor generation performance. In section \ref{method}, we propose a semantic-based initialization, which provides a better initial state for MCMC by leveraging the structure of DAE.
\section{Methodology}
\label{method}
In this section we describe how to build the ``semantic energy'' and ``texture energy''. Our model can be decomposed into two parts (figure\ref{Fig1}): a DAE with U-net structure and a semantic decoder. The overall learning procedure are trained via MLE.
\begin{figure}[H] 
\centering
\includegraphics[width=0.4\textwidth]{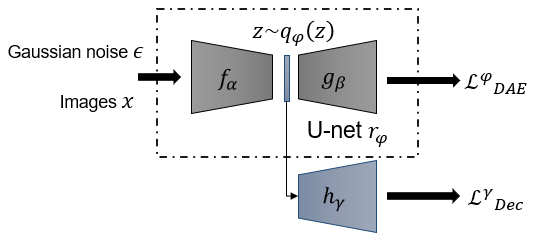} 
\caption{The architecture of Denoising-EBM. $r_{\varphi}=g_{\beta} \circ f_{\alpha}$ donate the DAE with U-net structure; $h_{\gamma}$ denote the semantic decoder mapping the code in latent space $\mathcal{Z}$ back to $\mathcal{X}$.} 
\label{Fig1} 
\end{figure}

\subsection{Energy Building}
Let $r_{\varphi}(\tilde{x},\sigma)=g_{\beta} \circ f_{\alpha}(\tilde{x},\sigma)$ denote the DAE, which maps from data space $\mathcal{X}$ to $\mathcal{X}$ with perturbed image $\tilde{x} \sim q_{\sigma}(\tilde{x} | x)$ as input, and $\varphi=(\alpha,\beta)$. And let $q_{\varphi}(z)$ denote the probability density of $z=f_{\alpha}(\tilde{x},\sigma)$. $h_{\gamma}(z)$ is a semantic decoder, mapping the semantic code $z$ in latent space $\mathcal{Z}$ back to $\mathcal{X}$. Here, the reason why we do not directly use $g_{\beta}$ as the decoder is that the U-net structure makes it impossible to ignore the skip connection from the encoder $f_{\alpha}$ in practice.

By Bayes rule, $q(x)=\frac{q(x | \tilde{x}) q(\tilde{x})}{q(\tilde{x} | x)}=\frac{q(x | \tilde{x}) q(\tilde{x})}{q_{\sigma}(\tilde{x} | x)} \propto \frac{q(x | \tilde{x})}{q_{\sigma}(\tilde{x} | x)}$ so that an we can get an estimated energy function from any given choice of $\tilde{x}$ through energy $U(x)  \approx \log q(x | \tilde{x})- \sigma_{const}$(constants related to $\sigma$)~\cite{ref44}. Thus, we can construct the energy function through the denoising reconstruction error:
\begin{equation}
\label{equ9}
U_{\varphi}^{c}(x, \sigma)=\frac{1}{2\sigma^2}\left\|x-r_{\varphi}(\tilde{x}, \sigma)\right\|^{2}.
\end{equation}
Actually, it will be defined as ``texture energy'' in latter. In this way, we vectorize the output of EBM and combine the U-net structure for capturing more detail information. However directly sampling $p(x)$ is difficult to converge. To alleviate this, we consider to model the latent distribution of noisy real data in the latent space of DAE. Here we assume $x \approx h_{\gamma}(z)$, so $p(z)=p_x(h_{\gamma}(z))|(\det(d h_{\gamma}/dz))|$. Then we ignore the Jacobian and define semantic energy by equation (8) and model the latent distribution by (9). It's experimentally proven to work. We supplement the definition:
$
p_{\varphi, \gamma}(x, z)=\frac{p_\gamma(x, z) e^{-U_{\varphi}^c(x)}}{Z(\varphi, \gamma)},
Z(\varphi, \gamma)=\int p_\gamma(x) e^{-U_{\varphi}^c(x)} d x
,p_\gamma(x, z)=p_{\gamma}(z)p_{\gamma}(x|z)$. Then we decompose the image energy by leveraging the structure of DAE. Theorem \ref{theorem2} reveals that the latent representation $z=f_{\alpha}(\tilde{x}, \sigma)$ of DAE captures as much information about training set $x \sim q(x)$. According to this, we named $\mathcal{Z}$ semantic space, where we can define the ``semantic energy'' $U_{\gamma}^{s}$:
\begin{align}
\label{equ10}
U_{\gamma}^{s}(z, \sigma)=\frac{1}{2\sigma^2}\left\|h_{\gamma}(z)-r_{\varphi}\left(\tilde{h}_{\gamma}(z), \sigma\right)\right\|^{2},\\ \nonumber
\tilde{h}_{\gamma}(z) = h_{\gamma}(z)+\epsilon,\quad \epsilon \sim \mathcal{N}\left(0, \sigma^{2} I_{d}\right).
\end{align}
Then the semantic model $p_{\gamma}(z)$ is formulated as an energy-based model to approximate $q(z|\tilde{x})$.
\begin{equation}
\label{equ11}
p_{\gamma}(z)=\frac{1}{Z(\gamma)} \exp \left(-U_{\gamma}^{s}(z, \sigma)\right).
\end{equation}
The semantic code $z$ will be assigned a lower energy due to the fact that decoder $h_{\gamma}$ map it back to the true data distribution, while the generated semantics will be penalized and given higher energy. Meanwhile, we define the ``texture energy'' $U_{\varphi}^{c}(x, \sigma)$  with equation \ref{equ9} in the data space $\mathcal{X}$.
Since the denoising reconstruction error of observe data is small, it will be assigned a lower energy, while the generated data under the model will be penalized and given higher energy.

Finally, we decompose the image energy into ``semantic energy'' and ``texture energy''. The ``semantic energy'' model the semantics in the latent space, and ``texture energy'' model the detail information of image in data space. In section \ref{mle} and \ref{geo},  we will show how they complement each other through two-stage MCMC sampling.
\begin{figure}[tb] 
\centering
\includegraphics[width=0.45\textwidth]{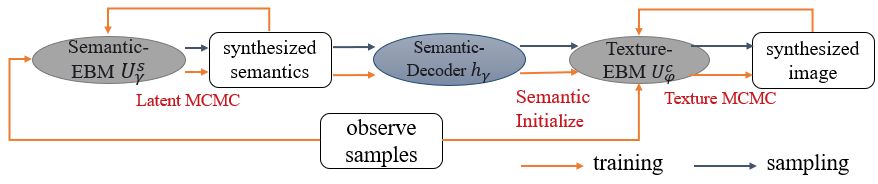} 
\caption{Overview of our framework.} 
\label{framework} 
\end{figure}
\subsection{Model Training}
\label{mle}
Given the training samples $x_{1}, x_{2},\ldots, x_{N} \sim q(x)$, consider the generative Markov structure $z \to x \to \tilde{x}$, we have $p(x,\tilde{x},z)=p(z)p(x | z)p(\tilde{x} | x)$, where $p(\tilde{x} | x)=q_{\sigma}(\tilde{x} | x)$. The parameters $\varphi$ and $\gamma$ of our model are trained by maximizing the lower bound of marginal log-likelihood on the perturbed data $\tilde{x}$:
\begin{align} \nonumber
\log p(\tilde{x})&=\mathbb{E}_{q(x, z | \tilde{x})} \log \frac{p(x, \tilde{x}, z)}{q(x, z | \tilde{x})}\\ \nonumber
&+D_{K L}(q(x, z | \tilde{x})|| p(x, z | \tilde{x})) \\
&\geq \mathbb{E}_{q(x, z | \tilde{x})} \log \frac{p(x, \tilde{x}, z)}{q(x, z | \tilde{x})}.\label{equ12}
\end{align}
Since the KL-divergence is non-negative, it's easy to verify that the bound is tight, then the object we want to optimize can be written as:
\begin{align}\nonumber
\qquad \qquad &\arg \max _{\varphi, \gamma} \mathbb{E}_{q(\tilde{x})} \mathbb{E}_{q(x, z | \tilde{x})}\left[\log \frac{p(x, \tilde{x}, z)}{q(x, z | \tilde{x})}\right] \\\nonumber
&\arg \max _{\varphi, \gamma} \mathbb{E}_{q(x, \tilde{x}, z)}[\log p(x, \tilde{x}, z)] \\\nonumber
&=\arg \max _{\varphi, \gamma} \mathbb{E}_{q(x, \tilde{x}, z)}\left[\log p_{\varphi, \gamma}(x, z) p(\tilde{x} \mid x)\right] \\\nonumber
&=\arg \max _{\varphi, \gamma} \mathbb{E}_{q(x, \tilde{x}, z)}\left[\log p_{\varphi, \gamma}(x, z)\right] \\\nonumber
&=\arg \max _{\varphi, \gamma}\underbrace{\mathbb{E}_{q(x, z)}\left[\log p_\gamma(x, z)\right]}_{\text{Semantic energy model}}\\
&-\underbrace{\mathbb{E}_{q(x)}\left[U_{\varphi}^c(x)-\log Z(\varphi, \gamma)\right]}_{\text{Texture energy model}}\label{equ13}
\end{align}
Here we ignore the term $H[q(x, z | \tilde{x})]$, because $q(x, z | \tilde{x})=q(x | \tilde{x})q(z | \tilde{x})$ does not vary with $\varphi, \gamma$. We use the gradient descent to maximize this object. 

We divide this into two parts: semantic energy model and texture energy model. We fix the parameter $\varphi$ when optimizing semantic energy model, and fix the parameter $\gamma$ when optimizing texture energy model. For semantic decoder, we optimize
\begin{align}\nonumber
&\arg \max _\gamma \mathbb{E}_{q(x, z)}\left[\log p_\gamma(x, z)\right]\\
&=\arg \max _\gamma \mathbb{E}_{q(x, z)}\left[\log p_\gamma(z)+\log p_\gamma(x \mid z)\right].
\label{equ14}
\end{align}
And the second term in equation \ref{equ13} is the expected decoding error of $h_{\gamma}$:
\begin{equation}
\label{equ8}
\mathcal{L}_{D e c}^{\gamma}=\mathbb{E}_{x \sim q(\tilde{x})}\left[\left\|h_{\gamma}\left(f_{\alpha}(\tilde{x}, \sigma)\right)-x\right\|^{2}\right].
\end{equation}

For DAE model, we can get
\begin{align}\nonumber
&\mathbb{E}_{q(x)} \nabla_{\varphi}\left[U_{\varphi}^c(x)+\log Z(\varphi, \gamma)\right]\\
&=\mathbb{E}_{q(x)}\left[\nabla_{\varphi} U_{\varphi}^c(x)\right]-\mathbb{E}_{p_{\varphi, \gamma}(x, z)}\left[\nabla_{\varphi} U_{\varphi}^c(x)\right].\label{equ15}
\end{align}
Expectation in equation \ref{equ14} and \ref{equ15} require MCMC sampling of the semantic marginal distribution $p_{\gamma}(z)$ in $\mathcal{Z}$ and the sample from $p_{(\varphi,\gamma)} (x,z)$. To approximate $p_{\gamma}(z)$, we iterate equation \ref{equ16} from the fixed initial Gaussian distribution. Specifically, since $\mathcal{Z}$ is low-dimensional, MCMC only need iterating a small number of $K$ steps, e.g., $K=20$,
\begin{equation}
\label{equ16}
z^{t+1}=z^{t}-\eta_1 \nabla_{z} U_{\gamma}^{s}\left(z^{t}, \sigma\right)+\sqrt{2 \eta_1} e_{t},\quad e_{t} \sim \mathcal{N}(0, I).
\end{equation}
To approximate $p_{\varphi}(x | z)$, we perform two-stage sampling. The first stage is iterating equation \ref{equ16} to get $z^K \sim \hat{p_{\gamma}}(z)$($\hat{p}$ denote the approximate distribution obtained by MCMC), and the second stage is iterating equation \ref{equ17} from the initial state $x^0=h_{\gamma}(z^K)$.
\begin{equation}
\label{equ17}
x^{t+1}=x^{t}-\eta_2 \nabla_{x} U_{\varphi}^{c}\left(x^{t}, \sigma\right)+\sqrt{2 \eta_2} e_{t},\quad e_{t} \sim \mathcal{N}(0, I).
\end{equation}
It is worth noting that the initial state $x^0$ has actually generated what we roughly need, so we name it semantic image. But due to the simplicity of the semantic model, it lacks some detailed texture information, and the second stage is to supplement the information. The initial distribution of MCMC in the second stage is close to the model distribution that needs to be sampled, so we only need to run a fixed number of $T$ steps, e.g., $T=90$, and the training of the entire framework is also very stable.
The training and sampling procedure is summarized in Algorithm \ref{alg1}. And the overview of our framework is demonstrated in figure \ref{framework}.

\begin{algorithm}[ht]
\caption{Training and sampling algorithm}
\label{alg1}
\textbf{Input}: (1) Training examples $x_{1}, x_{2},\ldots, x_{N} \sim q(x)$, (2) numbers of MCMC steps $K$ and $T$ and step-size $\eta_1$, $\eta_2$, (3) reconstruction training steps $L$ during per epoch, (4) different standard deviations $\{\sigma_i\}_{i=1}^S$. \\
\textbf{Output}: Parameter $\varphi,\gamma,$ synthetic samples $Y$.\\
\begin{algorithmic}[1] %[1] enables line numbers
%\STATE Let $t=0$.
\WHILE{not converged}
\FOR{$l=1 \to L$}
\STATE randomly sample a mini-batch examples $\{x_i\}_{i=1}^n$\\
and $\{\sigma_i\}_{i=1}^n$.
\STATE compute $\Delta\varphi=\nabla_{\varphi}\mathcal{L}_{D A E}^{\varphi}$, $\Delta\gamma=\nabla_{\gamma}\mathcal{L}_{Dec}^{\gamma}$.
\STATE update $\varphi,\gamma$ based on $\Delta\varphi,\Delta\gamma$ using Adam optimizer.
\ENDFOR
\FOR{$t=0 \to K-1$}
\STATE $z^{t+1}=z^{t}-\eta_1 \nabla_{z} U_{\gamma}^{s}\left(z^{t}, \sigma\right)+\sqrt{2 \eta_1} e_{t}$.
\ENDFOR
\STATE let $x^0=h_{\gamma}(z^K)$.
\FOR{$t=0 \to T-1$}
\STATE $x^{t+1}=x^{t}-\eta_2 \nabla_{x} U_{\varphi}^{c}\left(x^{t}, \sigma\right)+\sqrt{2 \eta_2} e_{t}$.
\ENDFOR
\STATE $\Delta \gamma=\sum_{i=1}^{n}\left[\nabla_{\gamma} U_{\gamma}^{s}\left(z_{i}, \sigma_{i}\right)\right]$
\STATE $\qquad-\sum_{i=1}^{n}\left[\nabla_{\gamma} U_{\gamma}^{s}\left(z_{i}^{K}, \sigma_{i}\right)\right]$.
\STATE $\Delta \varphi=\sum_{i=1}^{n}\left[\nabla_{\varphi} U_{\varphi}^{c}\left(x_{i}, \sigma_{i}\right)\right]$
\STATE $\qquad-\sum_{i=1}^{n}\left[\nabla_{\varphi} U_{\varphi}^{c}\left(x_{i}^{T}, \sigma_{i}\right)\right]$.
\STATE update $\varphi,\gamma$ based on $\Delta\varphi,\Delta\gamma$ using Adam optimizer.
\ENDWHILE
\STATE $Y=x^T$.
\STATE \textbf{return} $\varphi, \gamma, Y$.
\end{algorithmic}
\end{algorithm}

\subsection{Geometric Interpretation}
\label{geo}
\begin{figure}[t]
\centering
\includegraphics[width=0.4\textwidth]{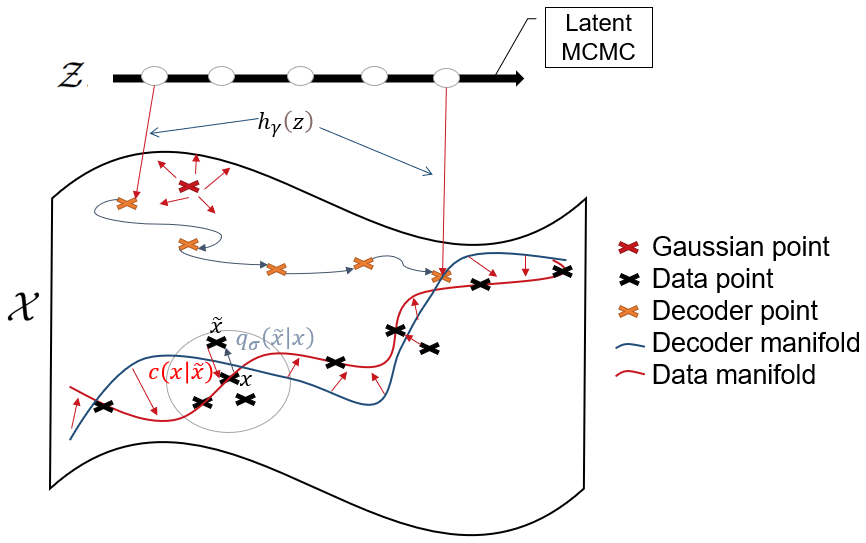} % Reduce the figure size so that it is slightly narrower than the column.
\caption{An manifold interpretation of our framework. Suppose that data point concentrates on the low-dimensional manifold embedded in high-dimensional input space $\mathcal{X}$. Through MCMC in latent space $\mathcal{Z}$, the decoder point moves to decoder manifold gradually. Then the MLE learning is the process of evolving from decoder manifold to data manifold.}
\label{fig2}
\end{figure}
We adopt the common assumption that high-dimensional data lie on a low-dimensional manifold. With this assumption, the process of model training can be visualized in figure \ref{fig2}. For convenience, we suppose that data point concentrates on the 1-dimensional curve embedded in 2-dimensional space $\mathcal{X}$. The Denoising-EBM learn a random operator $c(x | \tilde{x})=r_{\varphi}(\tilde{x},\sigma)$ that maps corrupted data $\tilde{x}$ to data manifold. When $\sigma$ is large, $\tilde{x}$ will be outside and farther from the manifold, and  $c(x | \tilde{x})$ should learn to make larger steps to reach the manifold. When $\sigma$ is small, $\tilde{x}$ will fall near the manifold, $c(x | \tilde{x})$ would make small steps. 

Importantly, according to ~\cite{ref41}, the range of noise level will affect the effective region of $c(x | \tilde{x})$, and determine whether the MCMC can converge to the data manifold. So we need to pick noises with different levels and make them spaced narrow enough to avoid blank areas in the data space as much as possible. Then we consider the MCMC of DAE, similiar to score-based model, the generate samples would move toward the learned manifold along the estimated score. When the model is optimal and sigma is small enough, combined with theorem \ref{theorem1} we have:
\begin{equation}
\label{equ19}
\nabla_{x} U_{\varphi}^{c}(x, \sigma) \approx \nabla_{x} \log p(x)-\frac{\partial r(x)}{\partial x}\left(\nabla_{x} \log p(x)\right).
\end{equation}
And when the density moves on the manifold learned by DAE, the reconstruction remains $r(x)\approx x$, which corresponds to the eigenvalues of $\frac{\partial r(x)}{\partial x}$ that are near $1$. Therefore, MCMC tend to converge to the data manifold. However, as shown in figure \ref{fig2}, if the initial state in $\mathcal{X}$ is random gaussian which is usually outside the effective region of $c(x | \tilde{x})$, it may not know how to get back to the high-density region. And the MCMC will have difficulty converging or converge to a local trap. Thus, in order to solve this problem, we introduced two-stage sampling. Let’s denote the decoder manifold $\mathcal{M}$ to be:
\begin{equation}
\label{equ20}
\mathcal{M}=supp(h_{\gamma}\left(z^{K}\right)), z^{K} \sim \hat{p}_{\gamma}(z).
\end{equation}
Sampling through the first stage, the initial state $x^0 \in \mathcal{M}$ already has a small reconstruction error. So it will locate near the learned manifold and $c(x | \tilde{x})$ can determine the direction and step size more accurately.

Finally, the training process can be interpreted as the following two procedures: a) training semantic energy model brings $\mathcal{M}$ closer to the manifold learned by DAE, b) training the texture energy model moves the DAE manifold close to data manifold. We can also treat $\mathcal{M}$ as a constraint to reduce the mode of MCMC in $\mathcal{X}$. Similar work can be found in~\cite{ref46}, but for a different purpose they ignore training in the latent space and lack sufficient semantic information for high-dimensional data.

\subsection{Compared with VAE-based EBM}
In section \ref{related}, we mentioned that some work combined VAE with EBM to make MCMC sampling more efficient, but few works consider DAE in this regard. Here, we present some theoretical analysis about the latent space on DAE and VAE. Similarly, let $r_{\varphi}=g_{\beta} \circ f_{\alpha}$ denote the VAE, where we have:
 \begin{align}
 \label{equ21} \nonumber
 \mathcal{L}(\alpha, \beta)&=\mathbb{E}_{q(x)}[\mathbb{E}_{q_\alpha(z | x)}\left[\log q_\beta(x | z)\right]\\
 &-K L\left(q_\alpha(z | x) \| p(z)\right)].
 \end{align}
The second term in equation \ref{equ21} imposes a regularizer over latent codes, may potentially result in very poor learned representations ~\cite{ref45}.
\begin{align} \nonumber
&\mathbb{E}_{q(x)} K L\left(q_\alpha(z | x)|| p(z)\right)=\int q_\alpha(x, z) \log \frac{q_\alpha(z | x)}{p(z)} d x d z\\\nonumber
&\geq \int q_\alpha(x, z) \log \frac{q_\alpha(z | x)}{p(z)} d x d z-K L\left(q_\alpha(z)|| p(z)\right)\\\nonumber
&=\int q_\alpha(x, z)\left[\log \frac{q_\alpha(z | x)}{p(z)}-\log \frac{q_\alpha(z)}{p(z)}\right] d x d z\\\nonumber
&=\int q_\alpha(x, z) \log \frac{q_\alpha(z | x)}{q_\alpha(z)} d x d z\\
&=I_{q}(x, z)\label{equ22}
\end{align}
From equation \ref{equ22}, we can observe that training a VAE also minimizes the upper bound of the mutual information between the representations and input data. Conversely, We prove that training a DAE maximizes the lower bound of the mutual information in Theorem \ref{theorem2}. That means modelling latent energy with DAE may be sounder for it can carry more information from the input data.
\subsection{Connection with denoising score matching}
Modeling and sampling the data distribution through denoising process results in a natural connection with denoising score matching~\cite{ref20,ref40}. Taking the gradient of energy function in \ref{equ9}, we get:
\begin{equation}
\label{equ23}
\nabla_{x} U_{\varphi}(x, \sigma)=\frac{1}{\sigma^2}(x-r(\tilde{x}))-\frac{1}{\sigma^2}(x-r(\tilde{x}))\nabla_{x}r(\tilde{x}).
\end{equation}
And the denoising score matching is to optimize:
\begin{equation}
\label{equ24}
\mathbb{E}_{p(x,\tilde{x})}\left[\left\|s_{\theta}(\tilde{{x}}, \sigma)+\frac{\tilde{x}-x}{\sigma^2}\right\|^2\right].
\end{equation}
where the $s_{\theta}(x)$ is the score of x. Theoretically, it needs approximating the gradient of the log-probability function, which should be a conservative vector field. However, directly parameterizing the score with unconstrained models are not guaranteed to output a conservative vector field. In contrast, the explicit energy-based model seems to have a theoretical advantage, since it models the score by explicitly differentiating a parameterized energy function ~\cite{ref29}.
In equation \ref{equ23}, the first term is equivalent to the score in equation \ref{equ24}, and the second term can viewed as the constraint term that guarantees the score model $s_{\theta}(x)$ to be a conservative vector field.
\section{Experiment}
\begin{figure*}
\centering
\includegraphics[width=0.8\textwidth]{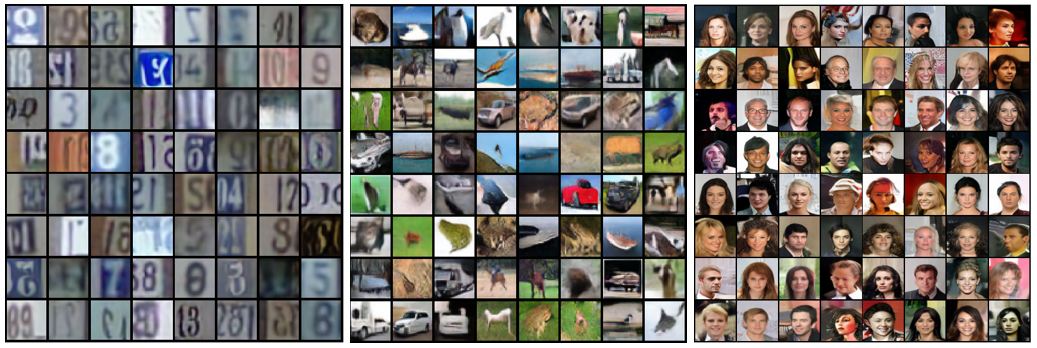} 
\caption{Randomly generated images. Left: SVHN samples. Mid: Cifar-10 samples. Right: CelebA 64$\times$64 samples.}
\label{fig3}
\end{figure*}
Experimental setting:  We used SVHN~\cite{ref68}, CIFAR-10~\cite{ref47}, CelebA 64$\times$64~\cite{ref48} in our experiments. For DAE and semantic decoder, we adopt the structure of ResNet~\cite{ref51,ref52}, where DAE also combines U-net. On hyperparameter selection, we choose $L=3, K=20, T=90$ for SVHN and CIFAR-10, $L=3, K=40, T=100$ for CelebA. Step-size $\eta_1=0.1$, $\eta_2=0.01$. And we fixed $S=128$ for different standard deviations with a geometric sequence from $\sigma_1=1$ to $\sigma_S=0.01$. All training experiments use Adam optimizer with learning rate of $5\times10^{-5}$. All datasets were scaled to $[-1,1]$ and flipped randomly with $p=0.5$ for pre-processing during training.
\subsection{Unconditional Image Generation} 
\begin{figure*}
\centering
\includegraphics[width=0.8\textwidth]{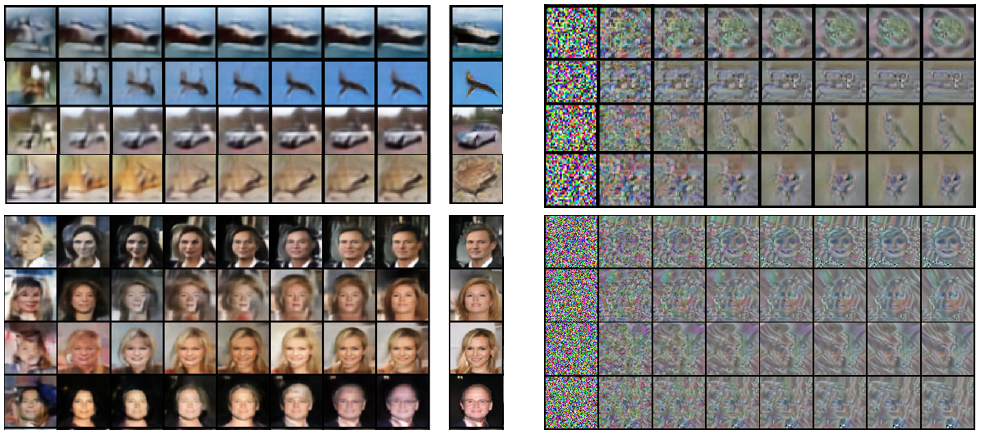} 
\caption{Left: Trajectory of semantic sampling from $h_{\gamma}(z^0)$ to $h_{\gamma}(z^K)$ (for every 5 steps on CIFAR-10, 8 steps on CelebA), $z^0 \sim \mathcal{N}(0, I)$. The last images are two-stage samples for comparison. Right: Trajectory of texture sampling without semantic initialization for every 100 steps, $x^0 \sim \mathcal{N}(0, I)$.}
\label{fig5}
\end{figure*}

Figures \ref{fig3} shows random samples from the two-stage MCMC sampling by our proposed model. We evaluated $50k$ synthetic samples after training with Inception Score (IS)~\cite{ref53} and FID~\cite{ref54} as quantitative metrics.  Table \ref{table1} and \ref{table2} summarize the quantitative evaluations on CIFAR-10 and Celeba 64 $\times$ 64. On CIFAR-10, our model achieves an FID of 21.24, an Inception Score of 7.86. On Celeba 64 $\times$ 64, our model achieves an FID of 14.1. These results are comparable to GAN-based methods and outperforms most energy-based methods. 
\begin{table}[ht]
\centering
%\resizebox{.95\columnwidth}{!}{
\begin{tabular}{ccc}
\toprule[1.5pt]
Model & IS$\uparrow$ & FID$\downarrow$ \\
\midrule[0.5pt]
\textbf{GAN-based} & & \\
\midrule[0.5pt]
WGAN-GP~\cite{ref58} & 7.86 & 36.4 \\
SNGAN~\cite{ref57} & 8.22 & 21.7  \\
EBGAN~\cite{ref59} & 6.40 & 37.11  \\
\midrule[0.5pt]
\textbf{Score-based} & & \\
\midrule[0.5pt]
NCSN~\cite{ref20} & 8.87 & 25.32 \\
DDPM~\cite{ref28} & $\mathbf{9.03}$ & $\mathbf{7.76}$ \\
\midrule[0.5pt]
\textbf{Other Likelihood Models} & & \\
\midrule[0.5pt]
PixelCNN~\cite{ref21} & 4.60 & 65.93  \\
Glow~\cite{ref62} & 3.92 & 48.9  \\
NVAE~\cite{ref63} & 5.51 & 51.67 \\
\midrule[0.5pt]
\textbf{Energy Based Model} & & \\
\midrule[0.5pt]
IGEBM~\cite{ref23} & 6.02 & 40.58 \\
JEM~\cite{ref26} & $\mathbf{8.76}$ & 38.4 \\
EBM+VAE~\cite{ref67} & 6.65 & 36.2  \\
CoopNets~\cite{ref69} & 6.55 & 33.61 \\
MDSM~\cite{ref41} & 8.31 & 31.7 \\
EBM-BB~\cite{ref56} & 7.45 & 28.63 \\
\textbf{Denoising-EBM} & 7.86 & $\mathbf{21.24}$ \\
\bottomrule[1.5pt]
\end{tabular}
\caption{FID and inception scores on CIFAR-10.}
\label{table1}
\end{table}
\begin{table}[ht]
\centering
%\resizebox{.95\columnwidth}{!}{
\begin{tabular}{cc}
\toprule[1.5pt]
    Model & FID$\downarrow$ \\
\midrule[0.5pt]
    \textbf{Denoising-EBM} & $\mathbf{14.1}$ \\
    SNGAN~\cite{ref57} & 50.4 \\
    NAVE~\cite{ref63} & 14.74  \\
    NCSN~\cite{ref20} & 25.3 \\
    EBM-Triangle~\cite{ref65} & 24.7  \\
    BiDVL~\cite{ref70} & 17.24 \\
\bottomrule[1.5pt]
\end{tabular}
\caption{FID on CelebA $64 \times 64$.}
\label{table2}
\end{table}

Based on our proposed framework, we also perform a decouple analysis of ``semantic energy'' and ``texture energy''. For ``semantic energy'', figure \ref{fig5} shows the fast MCMC trajectory of semantic image. We find that good semantic samples can be generated in the semantic space $\mathcal{Z}$ after a few steps. For celeba, 40 MCMC steps can synthesis semantically reasonable images. And for CIFAR-10, 20 steps are enough. However, it is obvious that the semantic image lacks some texture information compared with two-stage samples on the right. And for ``texture energy'', figure \ref{fig5} also shows trajectory of MCMC sampling on ``texture energy'' from gaussian noise. As described in section \ref{method}, MCMC sampling in pixel sapce $\mathcal{X}$ will consume long mixing time without semantic initialization. And it can be found that ``texture energy'' does model the texture information (e.g., eyes, hair) of the image data. This validates our original idea that we always notice the semantic representation of an image first and then observe its texture. And we can simulate this behavior with two-stage MCMC sampling.

The sampling time is mainly focused on MCMC mixing. ~\cite{ref20,ref28} usually require $\ge$ 1000 MCMC steps to synthesize high-fidelity images. In our framework, we only run 35 MCMC steps in latent space and 100 MCMC steps in data space to generate CIFAR-10 samples. And it only takes 8.67 seconds to generate 64 samples, while NCSN ~\cite{ref20} takes 120.11 seconds in the same setting.
\subsubsection{Ablation Studies}
We conduct ablation studies to isolate the influences of different noise settings on our model. Here, we compute fast FID (1000 samples) on CIFAR-10 as quantitative metric. In table \ref{table3}, we show the result of different noise levels $\sigma$ and interval density $S$.

When the $\sigma$ is too large, it is difficult for our model to converge to the optimal state, and the model will learn a ``blurry'' version of the sample distribution, which directly lead to the MCMC sampling both in $\mathcal{Z}$ and $\mathcal{X}$ unable to obtain enough image information. When $\sigma$ is small, any training examples of the our model will always fall near the real data manifold. When the sampling algorithm falls into the region other than the manifold, it may not know how to get back to the high-density region. And if the interval is too loose, the distribution of noisy training samples will have large gaps in the data space. When the sampling algorithm falls into the gaps, it may also not know how to get back to the high-density region. 
\begin{table}[ht]
\centering
%\resizebox{.95\columnwidth}{!}{
\begin{tabular}{ccc}
\toprule[1.5pt]
     & Setting & FID$\downarrow$ \\
\midrule[0.5pt]
\multirow{4}{4em}{noise level}
    & $\sigma_1=0.01,\sigma_{128}=3$ & 66.57 \\
    & $\sigma_1=0.01,\sigma_{128}=2$ & 60.81 \\
    & $\sigma_1=0.01,\sigma_{128}=1$ & $\mathbf{50.02}$  \\
    & $\sigma_1=0.01,\sigma_{128}=0.5$ & 55.23 \\
\midrule[0.5pt]
\multirow{3}{4em}{interval density}
    & $S=64$ & 60.02 \\
    & $S=128$ & $\mathbf{50.02}$ \\
    & $S=256$ & 51.10  \\
\bottomrule[1.5pt]
\end{tabular}
\caption{FID on different noise levels and interval density.}
\label{table3}
\end{table}

\subsection{Out-of-distribution Detection}
We evaluate our energy function by assessing its performance on out-of-distribution detection (OOD) task, which requires models to assign high value on the data manifold and low value on all other regions and can be consider as a proxy of the log probability. Specifically, we use the model trained on CIFAR10 and regard $-U_{\varphi}^{c}\left(x, \sigma\right)$ as the critic of samples. In order to prevent the noise disturbance from affecting the real distribution too much, we choose $\sigma=0.05$. We use SVNH, CelebA, CIFAR100 and interpolations of separate CIFAR10 images as OOD distributions. Following~\cite{ref23,ref24}, we construct OOD metric by area under the ROC curve (AUROC). In table \ref{table6}, our model has significantly higher performance than most of baseline, justifying that our energy function effectively fits the data likelihood and penalizes the likelihood of non-data-like regions. We noticed a severe drop in ROC scores on CIFAR100, probably because of the strong similarity between CIFAR10 and CIFAR100.
\begin{table*}[t]
\centering
%\resizebox{.95\columnwidth}{!}{
\begin{tabular}{cccccc}
\toprule[1.5pt]
  & & SVHN & CelebA & CIFAR100 & Interp. \\
\midrule[0.5pt]
\multirow{6}{5em}{\textbf{Unsupervised Training}}
    & NVAE~\cite{ref63} & 0.42 & 0.68 & 0.56 & 0.64\\
    & IGEBM~\cite{ref23} & 0.63 & 0.7 & 0.5 & 0.7\\
    & EBM-Triangle~\cite{ref65} & 0.68 & 0.56 & - & -\\
    & VAEBM~\cite{ref36} & 0.83 & 0.77 & $\mathbf{0.62}$ & 0.7\\
    & \textbf{Denoising-EBM} & $\mathbf{0.99}$ & $\mathbf{0.86}$ & $\mathbf{0.62}$ & $\mathbf{0.99}$\\
\midrule[0.5pt]
\multirow{2}{5em}{\textbf{Supervised Training}}
    & JEM~\cite{ref26} & 0.67 & 0.75 & 0.67 & 0.65\\
    & VERA~\cite{ref71} & 0.83 & 0.33 & 0.73 & 0.86\\
    & HDGE~\cite{ref66} & 0.96 & 0.8 & 0.91 & 0.82\\
\bottomrule[1.5pt]
\end{tabular}
\caption{: Out-of-distribution detection on SVHN, CelebA, CIFAR100, Interp. (linear interpolation between CIFAR-10 images) datasets. CIFAR-10 is the in-distribution. We compute the AUROC$\uparrow$ of negative texture energy $-U_{\varphi}^{c}\left(x, \sigma\right)$.}
\label{table6}
\end{table*}
\section{Conclusion}
In this paper, we propose a guidance on how to construct image energy. With the unique structure of DAE, we construct the ``energy'' by decomposing it into ``semantic energy'' and ``texture energy''. And more effective sampling can be performed by using two-stage sampling, leading to a more accurate density estimate. Finally, our model demonstrate strong performance on generating task. And in out-of-distribution detection our model has the top performance. In future work, we will try to generalize energy to continuous time and adopt our model on larger-scale images.

%%%%%%%%% REFERENCES
{\small
%\bibliographystyle{ieee_fullname}
%\bibliography{egbib}

\bibliographystyle{ieee_fullname}
\bibliography{reference}
}

\end{document}